\documentclass{article}

\usepackage{microtype}
\usepackage{graphicx}
\usepackage{subcaption}
\usepackage{booktabs}
\usepackage{float}

\usepackage{hyperref}

\usepackage[preprint]{icml2026}

\setcitestyle{numbers,square,sort&compress,citesep={,}}
\hypersetup{
  pdftitle={Where Does Reasoning Break? Step-Level Hallucination Detection via Hidden-State Transport Geometry},
  pdfauthor={Tyler Alvarez and Ali Baheri},
  pdfsubject={arXiv preprint},
  pdfkeywords={Hallucination Detection, Reasoning, Large Language Models, Contrastive PCA, Optimal Transport, Knowledge Distillation}
}

\usepackage{amsmath}
\usepackage{amssymb}
\usepackage{amsfonts}
\usepackage{mathtools}
\usepackage{amsthm}
\usepackage{nicefrac}
\usepackage{xcolor}
\usepackage{algorithm}
\usepackage{algorithmic}

\usepackage[capitalize,noabbrev]{cleveref}

\allowdisplaybreaks

\theoremstyle{plain}
\newtheorem{theorem}{Theorem}[section]
\newtheorem{proposition}[theorem]{Proposition}
\newtheorem{lemma}[theorem]{Lemma}

\theoremstyle{definition}

\theoremstyle{remark}

\newcommand{\R}{\mathbb{R}}
\newcommand{\E}{\mathbb{E}}
\newcommand{\Prob}{\mathbb{P}}
\newcommand{\Tr}{\operatorname{Tr}}
\newcommand{\op}{\operatorname{op}}

\newcommand{\med}{\operatorname{median}}
\newcommand{\MAD}{\operatorname{MAD}}
\newcommand{\cW}{\mathcal{W}}

\newcommand{\cC}{\mathcal{C}}

\definecolor{teachercolor}{RGB}{0,90,170}   
\definecolor{studentcolor}{RGB}{200,90,30}  
\newcommand{\teach}[1]{\textcolor{teachercolor}{#1}}
\newcommand{\stud}[1]{\textcolor{studentcolor}{#1}}

\icmltitlerunning{Where Does Reasoning Break? Hidden-State Transport Geometry}

\begin{document}

\twocolumn[
\icmltitle{Where Does Reasoning Break?\\
           Step-Level Hallucination Detection via Hidden-State Transport Geometry}

\icmlsetsymbol{equal}{*}

\begin{icmlauthorlist}
\icmlauthor{Tyler Alvarez}{rit}
\icmlauthor{Ali Baheri}{rit}
\end{icmlauthorlist}

\icmlaffiliation{rit}{Rochester Institute of Technology, Rochester, NY, USA}

\icmlcorrespondingauthor{Tyler Alvarez}{tma9531@rit.edu}
\icmlcorrespondingauthor{Ali Baheri}{akbeme@rit.edu}

\icmlkeywords{Hallucination Detection, Reasoning, Large Language Models,
              Contrastive PCA, Optimal Transport, Knowledge Distillation}

\vskip 0.3in
]

\printAffiliationsAndNotice{}  

\begin{abstract}
Large language models hallucinate during multi-step reasoning, but most existing detectors operate at the trace level: they assign one confidence score to a full output, fail to localize the first error, and often require multiple sampled completions. We frame hallucination instead as a property of the hidden-state trajectory produced during a single forward pass. Correct reasoning moves through a stable manifold of locally coherent transitions; a first error appears as a localized excursion in transport cost away from this manifold. We operationalize this view with a label-conditioned \emph{teacher} that builds a trace-specific contrastive PCA lens and scores each step with seven geometric transition features, and a deployable BiLSTM \emph{student} distilled from the teacher that operates on raw hidden states without inference-time labels. We prove that contrastive PCA is the optimal projection for a transport-separation objective between first-error and correct states, and that single-pass first-error localization holds whenever the first error creates a positive transport margin over preceding correct transitions. On ProcessBench, PRM800K, HaluEval, and TruthfulQA, both models outperform entropy-based, probing-based, and attention-based baselines in-domain; the teacher transfers stably across language models and datasets, while the student collapses under shift, a gap our distillation theory predicts. These results recast step-level hallucination detection as a problem of trajectory dynamics and identify the central obstacle to deployment: preserving the contrastive transport margin under distribution shift.
\end{abstract}

\section{Introduction}

Large language models (LLMs) now solve mathematical problems, multi-hop questions, and code generation tasks by producing long chains of reasoning steps~\citep{wei2022chain}. The same models also hallucinate within these chains, generating fluent steps that are nonetheless incorrect~\citep{ji2023survey, huang2024survey}, and a single early error typically propagates to a confidently wrong final answer. Detecting where reasoning first goes wrong, in a single forward pass, is therefore a prerequisite for trustworthy deployment of reasoning models.

Most existing hallucination detectors operate at the trace level. Token-level entropy and P(True) probing~\citep{kadavath2022language} and Bayesian uncertainty methods~\citep{gal2016dropout} are not calibrated for multi-step reasoning structure. Semantic entropy~\citep{kuhn2023semantic} and self-consistency~\citep{wang2023selfconsistency} sample many completions per prompt and aggregate, requiring multiple forward passes and producing one score per output. SelfCheckGPT~\citep{manakul2023selfcheckgpt} and INSIDE~\citep{chen2024inside} likewise reduce a full reasoning trace to a single confidence value. None of these methods localize the first error within a reasoning chain. Process reward models~\citep{lightman2023lets, wang2024mathshepherd, cobbe2021training} do score individual steps, but require expensive human annotation at training time and a separately trained verifier at inference. Hidden-state probes~\citep{burns2023discovering, azaria2023internal} classify the truthfulness of static factual statements, not the dynamics of an unfolding reasoning trace.

We propose to treat the sequence of hidden representations produced during a single forward pass as a trajectory in representation space, and to detect a hallucination as a localized excursion in transport cost away from the manifold of locally coherent transitions. We instantiate this view with two components. A label-conditioned \emph{teacher} uses step-level correctness labels to construct a trace-specific contrastive PCA lens and assigns each step a transport-instability score. The teacher is not deployable, since it requires labels at inference; rather, it is a diagnostic upper bound that quantifies the hidden-space signal. A \emph{student}, distilled from the teacher, learns to reproduce this score directly from raw hidden states, with no sampling, labels, or external verifier required at inference time.

\paragraph{Contributions.}
\begin{enumerate}
\item \textbf{A geometric framing of step-level hallucination detection.} We formulate hallucination detection as a problem of trajectory dynamics in hidden-state space, characterizing a first reasoning error as a localized transport excursion away from the manifold of locally coherent transitions. We instantiate this view with a label-conditioned \teach{teacher} that builds a trace-specific contrastive PCA lens and a deployable BiLSTM \stud{student} distilled from it that requires no labels, sampling, or external verifier at inference.

\item \textbf{Theoretical guarantees.} We prove that contrastive PCA is the optimal projection under a transport-separation objective between first-error and correct states (Theorem~\ref{thm:cpca-method}), establish a single-pass first-error localization bound under a transport margin assumption (Theorem~\ref{thm:localization-method}), and reduce \teach{teacher}--\stud{student} decision agreement to a margin-preservation condition (Proposition~\ref{prop:distillation-method}).

\item \textbf{Empirical findings.} On ProcessBench, PRM800K, HaluEval, and TruthfulQA, both models beat entropy-based, probing-based, and attention-based baselines in-domain. The \teach{teacher} transfers stably across models and datasets while the \stud{student} does not, a gap predicted by our distillation theory and identifying margin preservation under shift as the central deployment obstacle.
\end{enumerate}

\section{Related Work}

\textbf{Hallucination in LLMs.}
Hallucination, the generation of content that conflicts with source material or factual knowledge, is a well-surveyed phenomenon in language generation~\citep{ji2023survey} and large language models~\citep{huang2024survey}. For multi-step reasoning models, factual and reasoning-driven failures compound across steps, so the practically relevant target is to identify the \emph{step} at which the trace first deviates from a coherent reasoning path.

\textbf{Uncertainty and Hallucination Detection.}
Token-level entropy and P(True) probing~\citep{kadavath2022language} and Bayesian uncertainty methods~\citep{gal2016dropout} are not calibrated for multi-step reasoning. Semantic entropy~\citep{kuhn2023semantic} and self-consistency~\citep{wang2023selfconsistency} sample many completions and aggregate, requiring multiple forward passes and producing only trace-level scores. SelfCheckGPT~\citep{manakul2023selfcheckgpt} and INSIDE~\citep{chen2024inside} similarly reduce a full trace to one confidence value and do not localize where it first goes wrong.

\textbf{Process Supervision.}
Process reward models~\citep{lightman2023lets, wang2024mathshepherd, cobbe2021training} train supervised classifiers on human step-level correctness labels, and ProcessBench~\citep{zheng2024processbench} measures this ability directly. They rely on costly annotation, on a separately trained verifier at inference, and on dataset-specific labeling conventions, which limit transfer across models and tasks. Our deployable detector requires neither sampling nor a separate verifier.

\textbf{Probing, Interpretability, and Representation Engineering.}
Truth probes~\citep{burns2023discovering, azaria2023internal} train linear classifiers to predict whether a statement is true, activation patching~\citep{meng2022locating} localizes circuits responsible for factual recall, and representation engineering~\citep{zou2023representation} extracts conceptual directions for reading and steering high-level behaviors such as honesty. These methods analyze static representations of static claims; we instead model the \emph{trajectory} of hidden states across reasoning steps.

\textbf{Methodological Foundations.}
Contrastive PCA~\citep{abid2018contrastive} identifies low-dimensional directions enriched in a target distribution relative to a background; we use it in a trace-specific frame to expose first-error displacements. Optimal transport, in particular the squared 2-Wasserstein distance~\citep{villani2009optimal, peyre2019computational}, gives a clean transition score against the cloud of correct transitions, and the Davis-Kahan $\sin\Theta$ theorem~\citep{davis1970rotation} supplies the perturbation bounds for finite-sample stability. Recent adjacent work also supports geometry- and structure-aware views of reliability and representation learning~\citep{geometry_conformal_2026,conformal_scales_2025,multifidelity_temporal_2025,logic_vector_fields_2026,llm_contextual_bandit_2023}. Our deployable detector follows the standard knowledge distillation framework~\citep{hinton2015distilling}, although our analysis identifies \emph{margin} rather than mean error as the right distillation target for cross-domain transfer.

\textbf{Positioning.}
We depart from prior work along two axes. First, we score \emph{transitions} between reasoning steps rather than static states, using velocity, acceleration, and directional persistence in a contrastive lens, so that the detection signal is geometric rather than semantic. Second, we separate the question of whether a hidden-space signal exists (the label-conditioned teacher) from whether it can be recovered without labels at inference (the distilled student). This separation lets us prove guarantees for the geometric signal itself (Theorems~\ref{thm:cpca-method}--\ref{thm:localization-method}) and identify margin preservation under shift (Proposition~\ref{prop:distillation-method}) as the precise bottleneck for deployment.

\section{Methodology}
\label{sec:methodology}

GeoReason detects reasoning failures by treating a generated solution as a trajectory in the hidden space of a language model.  The central premise is geometric: correct reasoning may move through many semantic regions, but its step-to-step motion remains close to a stable manifold of locally coherent transitions; a first hallucination or reasoning error appears as a localized transport excursion away from this manifold.  Our method has two components (Figure~\ref{fig:architecture}).  First, a label-conditioned \emph{\teach{teacher}} constructs a contrastive geometric lens and converts hidden-state trajectories into transition-instability scores.  This \teach{teacher} is not deployable because it uses step labels to estimate its reference geometry.  Second, a \emph{\stud{student}} learns to reproduce the \teach{teacher}'s instability signal from raw hidden states alone, yielding a single-pass post-hoc detector for generated traces.

\begin{figure*}[!t]
  \centering
  \includegraphics[width=0.75\textwidth]{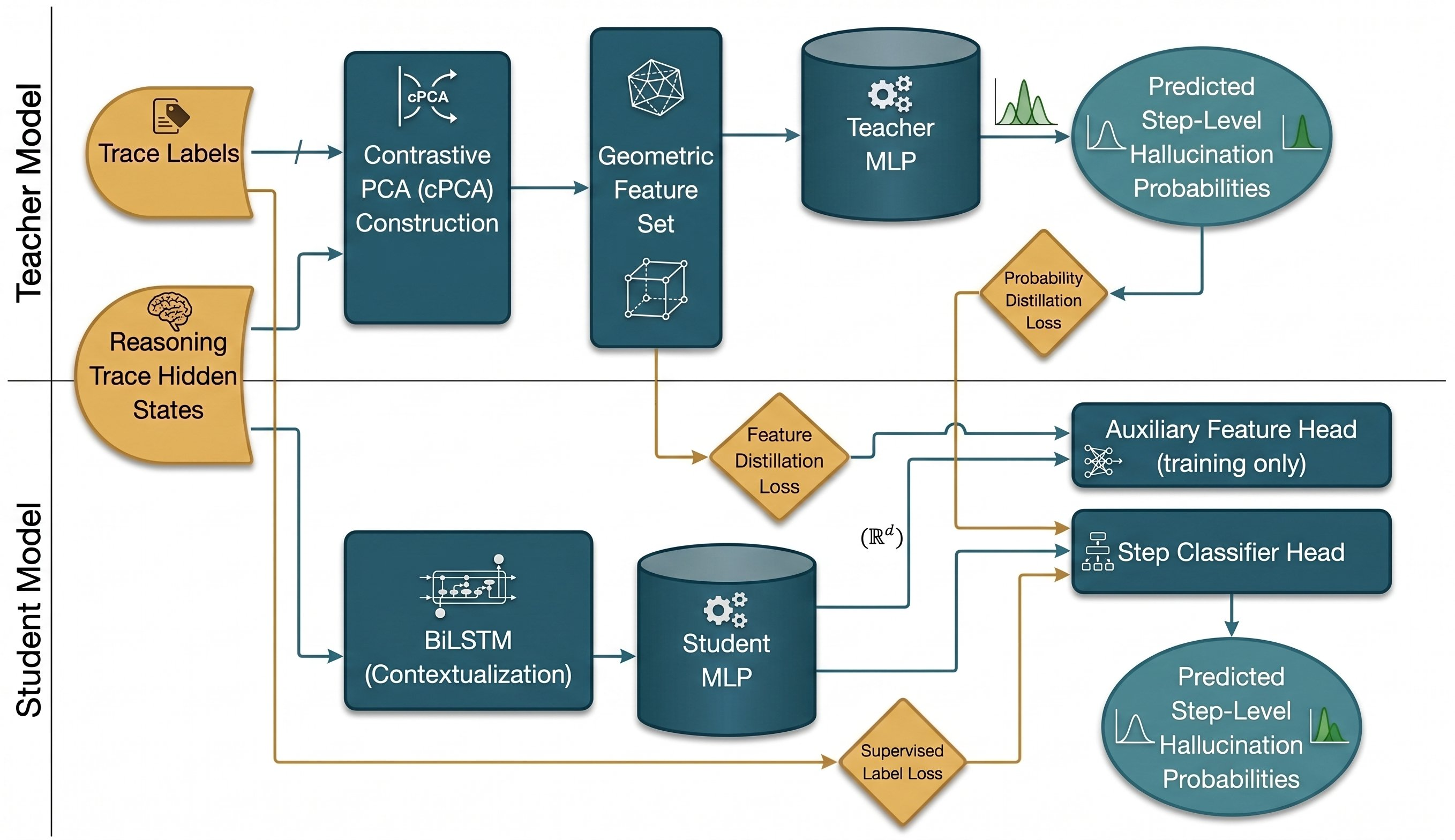}
  \caption{The GeoReason \teach{teacher}--\stud{student} architecture. The \teach{teacher} (top) uses step-level labels and reasoning-trace hidden states to construct a contrastive PCA (cPCA) projection, extracts a geometric feature set in this lens, and maps the features through an MLP to step-level hallucination probabilities. The \stud{student} (bottom) is a BiLSTM that contextualizes raw hidden states and feeds a step classifier head, trained from three signals: supervised step labels, probability distillation from the \teach{teacher}, and feature distillation through a training-only auxiliary head. At inference, the \stud{student} requires only hidden states.}
  \label{fig:architecture}
\end{figure*}

\subsection{Problem setup and step representations}
\label{subsec:problem}

For each prompt, let the model generate a trace $T=(s_1,\ldots,s_m)$, where each $s_t$ is a reasoning step, sentence, or phrase-level unit.  At a fixed transformer layer $\ell$, we map each step to one vector $h_t\in\R^d$ by mean pooling the token hidden states belonging to that step,
\begin{equation}
  h_t = \frac{1}{|I_t|}\sum_{j\in I_t} a_j^{(\ell)},
  \label{eq:step-pooling}
\end{equation}
where $I_t$ is the token index set for step $s_t$ and $a_j^{(\ell)}$ is the hidden state of token $j$ at layer $\ell$.  Other deterministic pooling rules, such as last-token pooling, can be used without changing the method.  During training and evaluation, a labeled trace has step labels $y_t\in\{0,1\}$, where $y_t=0$ denotes a correct step and $y_t=1$ denotes an incorrect step.  We define the first-error index as
\begin{equation}
\begin{aligned}
  &\tau = \min\{t: y_t=1\}, \\
  &\tau=\infty \text{ if the trace has no labeled error.}
\end{aligned}
\label{eq:first-error-index}
\end{equation}
For first-error localization, all steps after $\tau$ are treated as incorrect, since they are conditioned on a corrupted reasoning state even if their surface form later becomes plausible.  The objective is to learn scores $p_t\in[0,1]$ from a single generated trace such that $p_t$ is high exactly at and after the first error; the predicted first error is $\widehat\tau=\min\{t:p_t\geq\theta\}$, with no-error declared if no step crosses the threshold.

\subsection{Label-conditioned contrastive geometry}
\label{subsec:contrastive-geometry}

Raw hidden states contain many nuisance directions: prompt topic, syntax, answer length, and model-specific representation choices.  The \teach{teacher} therefore first converts each trace into a local coordinate system centered at its correct prefix.  Let $\cC=\{t:y_t=0\}$ be the correct steps in a labeled trace.  We compute
\begin{equation}
\begin{aligned}
  \bar h_0 &= \frac{1}{|\cC|}\sum_{t\in\cC} h_t, \\
  \sigma_0^2 &= \frac{1}{d|\cC|}\sum_{t\in\cC}\|h_t-\bar h_0\|_2^2, \\
  \widetilde h_t &= \frac{h_t-\bar h_0}{\sigma_0+\varepsilon}.
\end{aligned}
\label{eq:trace-normalization}
\end{equation}
This normalization makes the remaining signal relative to what the same model and prompt considered locally coherent before the error.

We then learn a contrastive PCA (cPCA) lens.  Let $P_0$ be the distribution of trace-normalized correct states and $P_1$ the distribution of first-error states, optionally including post-error states with smaller weight $\rho\in[0,1]$.  Denote their empirical means and covariances by $(\widehat\mu_0,\widehat C_0)$ and $(\widehat\mu_1,\widehat C_1)$.  GeoReason forms the contrastive transport matrix
\begin{equation}
  \widehat M_\alpha
  = (\widehat\mu_1-\widehat\mu_0)(\widehat\mu_1-\widehat\mu_0)^\top
    + \widehat C_1 - \alpha \widehat C_0,
  \qquad \alpha\geq 0,
  \label{eq:contrastive-matrix-method}
\end{equation}
where $\alpha$ controls how aggressively high-variance correct directions are suppressed.  The projection $U\in\R^{d\times k}$ is the matrix of the top $k$ eigenvectors of $\widehat M_\alpha$, and each step is represented in the geometric lens as
\begin{equation}
  z_t = U^\top \widetilde h_t\in\R^k.
  \label{eq:projected-state}
\end{equation}
When $\rho=0$, the lens isolates the first-error displacement; when $\rho=1$, it uses all incorrect steps and recovers the simpler implementation in which all post-error steps are assigned the incorrect class.  In practice, intermediate $\rho$ values prevent re-stabilized post-error states from diluting the first-error direction.

\subsection{Transition features and the \teach{teacher} detector}
\label{subsec:teacher}

The cPCA projection exposes where the trajectory leaves the correct manifold; transition features describe how it leaves.  Let
\begin{equation}
  \Delta z_t = z_t-z_{t-1},
  \qquad
  \Delta^2 z_t = z_t-2z_{t-1}+z_{t-2},
  \label{eq:velocity-acceleration}
\end{equation}
with zero padding for missing previous steps.  GeoReason uses the following feature block:
\begin{equation}
\begin{aligned}
  x_t &= \big[z_t,\; r_t,\; \bar r_t,\; v_t,\; a_t,\; e_t,\; d_t\big], \\
  r_t &= \|z_t\|_2, \quad
  \bar r_t = \frac{r_t-\med_{u\in\cC} r_u}{\MAD_{u\in\cC}(r_u)+\varepsilon},\\
  v_t &= \|\Delta z_t\|_2, \quad a_t = \|\Delta^2 z_t\|_2,\\
  e_t &= \frac{1}{w}\sum_{j=t-w+1}^{t}\!\big(v_j^2+a_j^2\big), \\
  d_t &= \frac{\langle \Delta z_t,\Delta z_{t-1}\rangle}
              {(\|\Delta z_t\|_2+\varepsilon)(\|\Delta z_{t-1}\|_2+\varepsilon)}.
\end{aligned}
\label{eq:feature-block}
\end{equation}
Here $r_t$ and $\bar r_t$ measure position in the contrastive space, $v_t$ and $a_t$ measure local motion, $e_t$ smooths transient noise, and $d_t$ distinguishes coherent continuation from an abrupt change of direction.  A lightweight MLP \teach{teacher} $\teach{f_\theta}$ maps $x_t$ to a probability
\begin{equation}
\begin{aligned}
  \teach{p_t^T} &= \sigma(\teach{f_\theta(x_t)}), \\
  \teach{\mathcal{L}_T(\theta)} &= -\sum_{i,t}
  \big[y_{it}\log \teach{p_{it}^T}\\
  &\qquad +(1-y_{it})\log(1-\teach{p_{it}^T})\big].
\end{aligned}
\label{eq:teacher-loss}
\end{equation}
The \teach{teacher} is best understood as a diagnostic upper bound on the hidden-space signal: it uses labels to construct the correct-step reference frame and the contrastive geometry, and is therefore not a valid standalone hallucination detector at deployment time.

\begin{algorithm}[H]
\caption{GeoReason \teach{teacher}: label-conditioned geometric instability}
\label{alg:teacher}
\begin{algorithmic}
\REQUIRE Labeled traces $\{(h_{i1:m_i},y_{i1:m_i})\}_{i=1}^n$, cPCA dimension $k$, background weight $\alpha$, window $w$.
\FOR{each trace $i$}
  \STATE $\cC_i\gets\{t:y_{it}=0\}$ and compute $\bar h_{i0},\sigma_{i0}$ from Eq.~\eqref{eq:trace-normalization}.
  \STATE Normalize each step: $\widetilde h_{it}\gets(h_{it}-\bar h_{i0})/(\sigma_{i0}+\varepsilon)$.
\ENDFOR
\STATE Estimate $(\widehat\mu_0,\widehat C_0)$ from all normalized correct steps and $(\widehat\mu_1,\widehat C_1)$ from first-error steps, optionally adding post-error steps with weight $\rho$.
\STATE Form $\widehat M_\alpha$ using Eq.~\eqref{eq:contrastive-matrix-method}; set $U\gets\operatorname{TopEig}_k(\widehat M_\alpha)$.
\FOR{each trace $i$ and step $t$}
  \STATE Project $z_{it}\gets U^\top\widetilde h_{it}$ and compute $x_{it}$ using Eq.~\eqref{eq:feature-block}.
\ENDFOR
\STATE Train the MLP \teach{teacher} $\teach{p^T_{it}}=\sigma(\teach{f_\theta(x_{it})})$ with Eq.~\eqref{eq:teacher-loss}.
\STATE \STATE \textbf{return} \teach{teacher} probabilities $\teach{p^T_{it}}$ and geometric features $x_{it}$.
\end{algorithmic}
\end{algorithm}

\subsection{Deployable \stud{student} by margin-preserving distillation}
\label{subsec:student}

The deployable model cannot use step labels to build $\cC$, $P_0$, or $P_1$ at inference.  We therefore distill the \teach{teacher} into a sequence model that takes only raw hidden states (Figure~\ref{fig:architecture}, bottom).  The \stud{student} is a BiLSTM followed by an MLP:
\begin{equation}
  c_{1:m}=\operatorname{BiLSTM}_\psi(h_{1:m}),
  \qquad
  \stud{p_t^S}=\sigma(\stud{g_\psi(c_t)}).
  \label{eq:student-architecture}
\end{equation}
The BiLSTM makes the detector post-hoc rather than online: it uses the whole generated trace, but it requires only one forward pass through the language model and no sampling, self-consistency, or external verifier.  We train the \stud{student} with a mixture of supervised step labels and soft \teach{teacher} targets,
\begin{equation}
\begin{aligned}
  \stud{\mathcal{L}_S(\psi)}
  &= \lambda\sum_{i,t}\operatorname{BCE}(y_{it},\stud{p_{it}^S}) \\
  &\quad + (1-\lambda)\tau_d^2\sum_{i,t}
   \operatorname{KL}\!\big(
      \operatorname{Bern}(\teach{q_{it}^T})\,\|\\
  &\qquad\quad \operatorname{Bern}(\stud{q_{it}^S})
    \big),\\
  \teach{q_{it}^T} &= \sigma(\operatorname{logit}(\teach{p_{it}^T})/\tau_d),\\
  \stud{q_{it}^S} &= \sigma(\operatorname{logit}(\stud{p_{it}^S})/\tau_d),
\end{aligned}
\label{eq:student-loss}
\end{equation}
where $\tau_d$ is the distillation temperature and $\lambda\in[0,1]$.  At test time, the \stud{student} returns the first threshold crossing
\begin{equation}
\begin{aligned}
  &\stud{\widehat\tau_S} = \min\{t:\stud{p_t^S}\geq\theta\}, \\
  &\stud{\widehat\tau_S}=\infty \text{ if no crossing occurs.}
\end{aligned}
\label{eq:student-first-error}
\end{equation}
Unless otherwise tuned on a validation split, we use $\theta=0.5$.

\begin{algorithm}[H]
\caption{GeoReason \stud{student}: deployable first-error detector}
\label{alg:student}
\begin{algorithmic}
\REQUIRE Training traces $\{h_{i1:m_i}\}_{i=1}^n$, optional labels $y_{it}$, \teach{teacher} probabilities $\teach{p^T_{it}}$, threshold $\theta$.
\STATE Train $\operatorname{BiLSTM}+\operatorname{MLP}$ \stud{student} with Eq.~\eqref{eq:student-loss}.
\STATE
\STATE \textbf{procedure} Infer($h_{1:m}$)
  \STATE $c_{1:m}\gets\operatorname{BiLSTM}(h_{1:m})$.
  \STATE $\stud{p_t^S}\gets\sigma(\stud{g(c_t)})$ for $t=1,\ldots,m$.
  \IF{$\max_t \stud{p_t^S} < \theta$} \STATE \STATE \textbf{return} no hallucination.
  \ELSE \STATE \STATE \textbf{return} hallucination with first-error estimate $\min\{t:\stud{p_t^S}\geq\theta\}$.
  \ENDIF
\STATE \textbf{end procedure}
\end{algorithmic}
\end{algorithm}

\subsection{Main theoretical results}
\label{subsec:theory}

We now justify the geometry used above.  The statements are intentionally assumption-explicit: GeoReason is guaranteed when the first semantic error induces a detectable transport-margin event in hidden-state trajectory space.  If a model makes an error while remaining hidden-state indistinguishable from correct trajectories, no geometry-only detector can be guaranteed to localize it.

\paragraph{Transport score for a reasoning transition.}
For an orthonormal projection $U$, define the augmented transition vector
\begin{equation}
  \phi_t(U)=\big[z_t,\; \Delta z_t,\; \Delta^2 z_t\big]\in\R^{3k}.
  \label{eq:augmented-transition-method}
\end{equation}
Let $R_0^U$ be the distribution of $\phi_t(U)$ over correct transitions.  For any positive semidefinite ground-cost matrix $A\succeq0$, define the point-to-cloud transport instability score
\begin{equation}
\begin{aligned}
  S_U(t)&=\cW_{2,A}^2(\delta_{\phi_t(U)},R_0^U) \\
  &:= \inf_{\pi\in\Pi(\delta_{\phi_t(U)},R_0^U)}
  \E_{(x,y)\sim\pi}(x-y)^\top A(x-y).
\end{aligned}
\label{eq:ot-score-method}
\end{equation}
Because one marginal is a point mass, the coupling is unique and
\begin{equation}
  S_U(t)=\E_{Y\sim R_0^U}(\phi_t(U)-Y)^\top A(\phi_t(U)-Y).
  \label{eq:ot-score-closed-form}
\end{equation}
Thus $S_U(t)$ is the cost of transporting the observed transition to the empirical cloud of correct transitions.  The features in Eq.~\eqref{eq:feature-block} are low-order summaries of this quadratic transport cost: position, velocity, acceleration, local energy, and directional persistence.

\begin{theorem}[cPCA maximizes a transport-separation objective]
\label{thm:cpca-method}
Let $X_0\sim P_0$ be a trace-normalized correct hidden vector and $X_1\sim P_1$ a trace-normalized first-error hidden vector, with means $\mu_0,\mu_1$ and covariances $C_0,C_1$.  For $U\in\R^{d\times k}$ with $U^\top U=I_k$, define
\begin{equation}
\begin{aligned}
  \Gamma(U)=\,&\E_{X_1}\cW_2^2(\delta_{U^\top X_1},U_\#P_0) \\
  &-\E_{X_0}\cW_2^2(\delta_{U^\top X_0},U_\#P_0),
\end{aligned}
\label{eq:transport-gap-method}
\end{equation}
where $U_\#P_0$ is the pushforward of $P_0$ under $U^\top$.  Then
\begin{equation}
\begin{aligned}
  \Gamma(U)&=\Tr\!\left(U^\top M U\right), \\
  M&=(\mu_1-\mu_0)(\mu_1-\mu_0)^\top+C_1-C_0.
\end{aligned}
\label{eq:theory-matrix-method}
\end{equation}
Consequently, the maximizer of $\Gamma(U)$ over all $k$-dimensional orthonormal projections is the top-$k$ eigenspace of $M$, and the optimal value is the sum of the top $k$ eigenvalues of $M$.
\end{theorem}

\noindent\textit{Proof sketch.}
For any fixed $x$, $\cW_2^2(\delta_{U^\top x},U_\#P_0)=\E_{Y\sim P_0}\|U^\top(x-Y)\|_2^2$.  Taking expectation over $X_1$ gives a projected second moment with covariance $C_1+C_0$ and mean shift $(\mu_1-\mu_0)(\mu_1-\mu_0)^\top$; taking expectation over $X_0$ gives $2\Tr(U^\top C_0U)$.  Subtracting yields Eq.~\eqref{eq:theory-matrix-method}.  The eigenspace claim follows from the Ky Fan variational principle.  Full proofs and finite-sample perturbation bounds are deferred to Appendix~\ref{app:proofs}.

Theorem~\ref{thm:cpca-method} explains why cPCA is the appropriate lens for GeoReason.  It selects directions that make first-error states far from the correct-state cloud while penalizing directions in which correct reasoning already has high variance.  If trace normalization removes the mean shift, the objective reduces to the usual contrastive covariance $C_1-C_0$; adding the parameter $\alpha$ in Eq.~\eqref{eq:contrastive-matrix-method} gives the generalized background-penalized form $C_1-\alpha C_0$.

\begin{theorem}[First-error localization under a transport margin]
\label{thm:localization-method}
Let $\tau$ be the first error and let $S(t)$ be the ideal transport score in Eq.~\eqref{eq:ot-score-method}.  Suppose that, for all $t<\tau$,
\begin{equation}
\begin{aligned}
  \Prob\{S(t)-\mu_c\geq u\}
  &\leq \exp\{-c\min(u^2/\nu^2,u/b)\}, \\
  &\quad \forall u>0,
\end{aligned}
\label{eq:correct-tail-method}
\end{equation}
for constants $c,\nu,b>0$.  Suppose also that the first error has margin
$\Prob\{S(\tau)\geq\mu_c+\gamma\}\geq1-\beta$ and that the empirical score $\widehat S(t)$ satisfies
\begin{equation}
  \Prob\left\{\max_{t\leq\tau}|\widehat S(t)-S(t)|\leq \gamma/4\right\}\geq1-\alpha .
  \label{eq:estimation-event-method}
\end{equation}
Then the threshold $\theta=\mu_c+\gamma/2$ and first crossing rule $\widehat\tau=\min\{t:\widehat S(t)\geq\theta\}$ obey
\begin{equation}
\begin{aligned}
  \Prob\{\widehat\tau=\tau\} \geq\;
  &1-\alpha-\beta\\
  &-(\tau-1)\exp\!\left[-c\min\!\left(\frac{\gamma^2}{16\nu^2},\frac{\gamma}{4b}\right)\right].
\end{aligned}
\label{eq:first-error-bound-method}
\end{equation}
\end{theorem}

\noindent\textit{Proof sketch.}
On the estimation event, the first-error score remains above $\mu_c+3\gamma/4$ whenever the margin event holds, and hence crosses $\theta$.  A false alarm before $\tau$ can occur only if some correct step has $S(t)\geq\mu_c+\gamma/4$.  Applying the sub-exponential tail bound and a union bound over $t<\tau$ gives Eq.~\eqref{eq:first-error-bound-method}.  See Appendix~\ref{app:proofs} for full details.

The result matches the empirical design goal: post-error states need not remain anomalous forever.  Localization only requires the first wrong step to create a transport excursion before any preceding correct transition does.

\begin{proposition}[Distillation preserves first-error decisions when margins survive]
\label{prop:distillation-method}
Let $\teach{s_T(t)}$ and $\stud{s_S(t)}$ be \teach{teacher} and \stud{student} scores on the same trace with common threshold $\theta$, and define the \teach{teacher} decision margin
\begin{equation}
  \teach{m_T}=\min_{1\leq t\leq m}|\teach{s_T(t)}-\theta|.
  \label{eq:teacher-margin}
\end{equation}
If $\max_t|\stud{s_S(t)}-\teach{s_T(t)}|\leq\varepsilon<\teach{m_T}$, then \teach{teacher} and \stud{student} assign identical labels to every step and return the same first-error index.  For random traces,
\begin{equation}
  \Prob\{\stud{\widehat\tau_S}\neq\teach{\widehat\tau_T}\}
  \leq
  \Prob\{\teach{m_T}\leq\varepsilon\}
  +\Prob\{\max_t|\stud{s_S(t)}-\teach{s_T(t)}|>\varepsilon\}.
  \label{eq:distillation-bound}
\end{equation}
\end{proposition}

Proposition~\ref{prop:distillation-method} identifies the main failure mode of the deployable \stud{student}.  Strong in-domain performance requires only small average distillation error, but robust first-error transfer requires preserving the \teach{teacher}'s \emph{margin} under shifts in model family, prompt distribution, and dataset style.  This is why we report both the label-conditioned \teach{teacher} and the deployable \stud{student}: the \teach{teacher} measures whether a contrastive transport signal exists in hidden space, while the \stud{student} measures how much of that signal can be recovered without inference-time labels.

\begin{figure*}[tb]
  \centering
  \includegraphics[width=\textwidth]{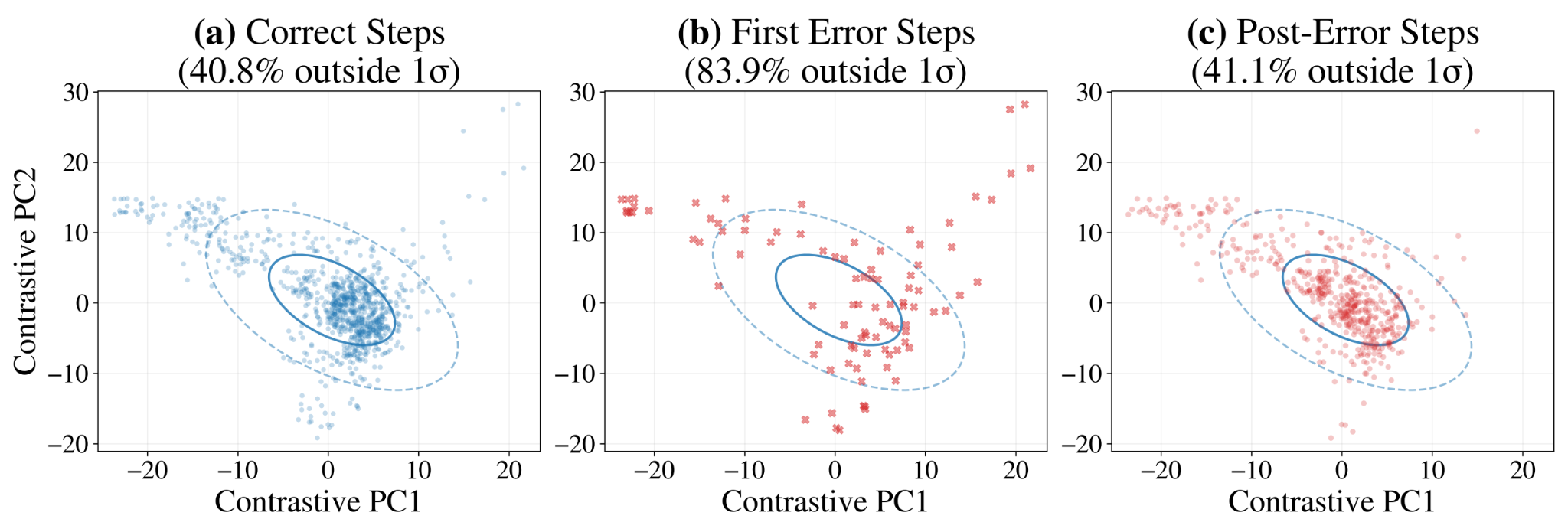}
  \caption{Hidden states projected into cPCA space for (a) correct, (b) first-error, and (c) post-error steps. First-error steps lie largely outside the correct-step distribution (83.9\% outside 1$\sigma$), while post-error steps partially overlap (41.1\%). This supports our view of hallucinations as trajectory deviations from a stable reasoning manifold.}
  \label{fig:trajectory-scatter}
\end{figure*}

\section{Experiments}
\label{sec:experiments}

We evaluate GeoReason on two tasks: step-level hallucination detection and first-error localization. Step-level detection is measured with AUROC, while first-error localization is measured by the accuracy of the first step whose score crosses a validation-selected threshold. The main in-domain results are reported in Table~\ref{tab:baseline_auroc} and Table~\ref{tab:baseline_first_error}.

\paragraph{Benchmarks and preprocessing.}
We use ProcessBench, PRM800K, HaluEval, and TruthfulQA because they cover process-level mathematical reasoning, annotated solution steps, generated hallucinations, and factual truthfulness. Each example is converted to a single ordered trace. When a benchmark provides step boundaries, we preserve them; otherwise, we split generated text at newline and sentence-level delimiters and discard empty fragments. Step labels are mapped to binary correctness. For localization, the first annotated incorrect step is treated as the first-error index and all later steps are evaluated as post-error states, matching the objective in Eq.~\eqref{eq:first-error-index}. Splits are prompt-level and stratified by dataset and error presence, so no reasoning trace appears in more than one split.

\paragraph{Hidden-state extraction and model settings.}
The cross-model experiments use one representative instruction-tuned model from each family: Qwen, Llama, and Mistral. For every generated trace, we run a single forward pass, extract the residual-stream hidden states from the final transformer block before the language-model head, and mean-pool tokens within each step as in Eq.~\eqref{eq:step-pooling}. Unless otherwise stated, the cPCA lens uses rank $k=16$, contrastive penalty $\alpha=1$, post-error weight $\rho=0.25$, and smoothing window $w=3$. Thresholds for first-error localization are tuned only on the validation split and then frozen for test evaluation.

\paragraph{Training details and baselines.}
The \teach{teacher} is a two-layer MLP over the feature block in Eq.~\eqref{eq:feature-block}. The \stud{student} is a two-layer BiLSTM with a step-classification head and a training-only auxiliary head for feature distillation. Both models are trained with AdamW, early stopping on validation AUROC, and prompt-level mini-batches. Baselines are evaluated under the same splits and hidden-state extraction protocol: TL-Entropy and TL-Perplexity use token-level likelihood statistics, Linear Probe trains a linear classifier on pooled step representations, and LLM-Check uses attention-derived scores. The \teach{teacher} should be read as an oracle diagnostic for the existence of a label-conditioned geometric signal; the \stud{student} is the deployable model because it does not use inference-time labels, sampling, or an external verifier.

Figure~\ref{fig:trajectory-scatter} visualizes the central phenomenon behind the method: the first error can produce a localized geometric excursion even when later post-error states move back toward the region occupied by correct steps. This supports the use of transition scores and first-crossing rules rather than a single trace-level endpoint confidence.

Across datasets, both the \teach{teacher} and \stud{student} models achieve strong performance relative to prior baselines. For step-level detection (Table~\ref{tab:baseline_auroc}), the \teach{teacher} attains the highest AUROC on three of the four datasets, including ProcessBench (91.0), HaluEval (94.0), and TruthfulQA (96.0), while the \stud{student} achieves the best performance on PRM800K (99.8). Baseline methods such as TL-Entropy, TL-Perplexity, Linear Probe, and LLM-Check generally perform worse across most datasets.

For first-error detection (Table~\ref{tab:baseline_first_error}), the \teach{teacher} achieves the highest accuracy on ProcessBench (68.7), PRM800K (88.4), and HaluEval (68.7), while the \stud{student} achieves the best performance on TruthfulQA (96.8) and PRM800K (92.9). Baselines again trail behind both proposed models in most settings, although the linear probe performs competitively on certain datasets such as TruthfulQA. Because the tables report point estimates, we interpret small differences cautiously and focus on the repeated pattern across benchmarks and generalization settings.

Overall, these results show that both the \teach{teacher} and \stud{student} models are effective for step-level hallucination detection and first-error localization across a range of benchmarks, consistently outperforming standard entropy-based, probing-based, and attention-based baselines.

\begin{table}[H]
    \caption{Step-level hallucination detection performance (AUROC) across benchmarks. Best results per column are in \textbf{bold}.}
    \label{tab:baseline_auroc}
    \centering
    \small
    \resizebox{\columnwidth}{!}{%
    \begin{tabular}{lcccc}
        \toprule
        \textbf{Methods} & \textbf{ProcessBench} & \textbf{PRM800K} & \textbf{HaluEval} & \textbf{TruthfulQA} \\
        \midrule
        \teach{Teacher (non-deployable)} & \textbf{91.0} & 98.5 & \textbf{94.0} & 96.0 \\
        \stud{Student (deployable)}      & 75.0 & \textbf{99.8} & 88.4 & \textbf{96.5} \\
        TL-Entropy               & 57.1 & 54.5 & 50.8 & 64.4 \\
        TL-Perplexity            & 51.2 & 45.8 & 48.4 & 67.1 \\
        Linear Probe             & 67.8 & 91.3 & 78.6 & 90.0 \\
        LLM-Check (attention)    & 61.9 & 48.0 & 55.7 & 69.8 \\
        \bottomrule
    \end{tabular}%
    }
\end{table}

\begin{table}[H]
    \caption{First-error detection accuracy across benchmarks. Best results per column are in \textbf{bold}.}
    \label{tab:baseline_first_error}
    \centering
    \small
    \resizebox{\columnwidth}{!}{%
    \begin{tabular}{lcccc}
        \toprule
        \textbf{Methods} & \textbf{ProcessBench} & \textbf{PRM800K} & \textbf{HaluEval} & \textbf{TruthfulQA} \\
        \midrule
        \teach{Teacher (non-deployable)} & \textbf{68.7} & 88.4 & 68.7 & 93.2 \\
        \stud{Student (deployable)}      & 34.4 & \textbf{92.9} & \textbf{78.6} & \textbf{96.8} \\
        TL-Entropy               & 46.3 & 43.8 & 42.3 & 68.5 \\
        TL-Perplexity            & 43.2 & 39.0 & 46.4 & 78.2 \\
        Linear Probe             & 24.4 & 68.8 & 68.7 & 89.3 \\
        LLM-Check (attention)    & 43.8 & 49.2 & 49.8 & 50.0 \\
        \bottomrule
    \end{tabular}%
    }
\end{table}

We evaluate cross-model generalization, where models are trained on a single LLM and evaluated on all three LLMs (Qwen, Llama, and Mistral). Tables~\ref{tab:cross_model_teacher} and~\ref{tab:cross_model_student} report results for the \teach{teacher} and \stud{student} models, respectively. We report step-level AUROC and first-error detection accuracy (in parentheses).

The \teach{teacher} demonstrates consistently strong performance across all train--test combinations (Table~\ref{tab:cross_model_teacher}). Step-level AUROC remains stable across models, ranging from 90.1 to 91.7, regardless of the training model. First-error detection accuracy is also consistent, with values between 72.2 and 78.3 across all settings.

The \stud{student} model exhibits strong performance when evaluated on the same model it was trained on (Table~\ref{tab:cross_model_student}). For example, training and testing on Qwen, Llama, and Mistral yields AUROC values of 93.6, 93.5, and 93.9, respectively, with corresponding first-error accuracies of 75.7, 74.4, and 75.3. However, performance varies substantially across different train--test combinations. In particular, when evaluated on models different from the training model, AUROC values range from 33.4 to 58.5, and first-error accuracy ranges from 27.0 to 35.4.

Overall, these results show that both models achieve strong in-domain performance, while cross-model evaluation reveals differences in performance consistency across training and testing configurations.

\begin{table}[H]
    \centering
    \caption{Cross-model generalization for the \teach{\textbf{Teacher}}. Models are trained on a single LLM and evaluated on all LLMs. We report step-level AUROC (top) and first-error accuracy (bottom in parentheses).}
    \label{tab:cross_model_teacher}
    \resizebox{\columnwidth}{!}{%
    \begin{tabular}{lccc}
        \toprule
        \textbf{Train $\downarrow$ / Test $\rightarrow$} & \textbf{Qwen} & \textbf{Llama} & \textbf{Mistral} \\
        \midrule
        Qwen    & 91.6 (\,77.5\,) 
                & 91.7 (\,77.8\,) 
                & 90.9 (\,75.1\,) \\
        Llama   & 90.7 (\,72.2\,) 
                & 91.7 (\,78.2\,) 
                & 91.2 (\,75.9\,) \\
        Mistral & 90.1 (\,73.4\,) 
                & 90.6 (\,74.6\,) 
                & 91.5 (\,78.3\,) \\
        \bottomrule
    \end{tabular}%
    }
\end{table}

\begin{table}[H]
    \centering
    \caption{Cross-model generalization for the \stud{\textbf{Student}}. Models are trained on a single LLM and evaluated on all LLMs. We report step-level AUROC (top) and first-error accuracy (bottom in parentheses).}
    \label{tab:cross_model_student}
    \resizebox{\columnwidth}{!}{%
    \begin{tabular}{lccc}
        \toprule
        \textbf{Train $\downarrow$ / Test $\rightarrow$} & \textbf{Qwen} & \textbf{Llama} & \textbf{Mistral} \\
        \midrule
        Qwen    & 93.6 (\,75.7\,) 
                & 39.1 (\,30.0\,) 
                & 35.4 (\,27.0\,) \\
        Llama   & 51.5 (\,31.9\,) 
                & 93.5 (\,74.4\,) 
                & 33.4 (\,28.4\,) \\
        Mistral & 48.8 (\,32.1\,) 
                & 58.5 (\,35.4\,) 
                & 93.9 (\,75.3\,) \\
        \bottomrule
    \end{tabular}%
    }
\end{table}

We evaluate cross-dataset generalization under a leave-one-out setup, where each method is trained on all datasets except one and evaluated on the held-out dataset. Table~\ref{tab:cross_dataset_results} shows that the \teach{teacher} maintains strong performance across all held-out datasets (AUROC 59.2 to 91.3, first-error 49.8 to 89.0), while the \stud{student} is consistently weaker on PRM800K, HaluEval, and TruthfulQA, with the gap reaching $\Delta=+39.9$ AUROC and $+40.1$ first-error accuracy.

\begin{table}[H]
    \centering
    \caption{Cross-dataset generalization under a leave-one-out setup. Each method is trained on all datasets except the held-out one. We report step-level AUROC and first-error detection accuracy. $\Delta$ denotes the performance gap (\teach{Teacher} $-$ \stud{Student}).}
    \label{tab:cross_dataset_results}
    \resizebox{\columnwidth}{!}{%
    \begin{tabular}{lcccccc}
        \toprule
         & \multicolumn{3}{c}{\textbf{Step-Level AUROC}} & \multicolumn{3}{c}{\textbf{First-Error Acc.}} \\
        \textbf{Held-out Dataset} & \teach{\textbf{Teacher}} & \stud{\textbf{Student}} & $\Delta$ & \teach{\textbf{Teacher}} & \stud{\textbf{Student}} & $\Delta$ \\
        \midrule
        ProcessBench & 59.2 & 62.5 & -3.3 & 54.2 & 38.7 & +15.5 \\
        PRM800K      & 69.0 & 51.3 & +17.7 & 49.8 & 9.7 & +40.1 \\
        HaluEval     & 86.7 & 58.3 & +28.4 & 69.5 & 59.8 & +9.7 \\
        TruthfulQA   & 91.3 & 51.4 & +39.9 & 89.0 & 50.4 & +38.6 \\
        \bottomrule
    \end{tabular}%
    }
\end{table}


\section{Analysis}
The teacher computes geometric features in a label-conditioned, trace-specific cPCA space and outperforms all baselines both in-domain and on out-of-domain models and datasets. Because it requires step labels at inference, we treat it not as a deployable detector but as a diagnostic upper bound that demonstrates the framing of hallucination as trajectory instability. The student, which removes this requirement, also outperforms the baselines in-domain but collapses to near-random AUROC under cross-model and cross-dataset shift.

We attribute this gap to what each model learns. The teacher captures a mechanistic instability signal that is intrinsic to the geometry of correct vs.\ first-error transitions; the student, operating in the full latent space without label-conditioned normalization, learns a representational signal that absorbs model- and dataset-specific encoding quirks. Improving deployment therefore requires distillation that preserves the teacher's transport margin rather than only its mean predictions.

\section{Conclusion}

GeoReason frames step-level hallucination detection as hidden-state trajectory geometry: a label-conditioned teacher exposes the geometric signal of a first error via trace-specific cPCA, and a deployable student distills this signal for single-pass detection from raw hidden states. We prove that cPCA is the optimal lens under a transport-separation objective (Theorem~\ref{thm:cpca-method}), that localization holds whenever a transport margin exists (Theorem~\ref{thm:localization-method}), and that teacher-student agreement reduces to margin preservation (Proposition~\ref{prop:distillation-method}). The teacher transfers across models and datasets while the student does not, identifying margin preservation under shift, rather than detection of the geometric signal, as the central deployment obstacle.

\clearpage
\bibliography{ref}
\bibliographystyle{icml2026}

\newpage
\appendix
\onecolumn

\section{Proofs and Additional Theoretical Analysis}
\label{app:proofs}

This appendix gives the formal details for the theoretical section.  The statements are intentionally assumption-explicit: GeoReason can be guaranteed only when the first reasoning error induces a measurable transport-margin event in hidden-state trajectory space.  If a model makes a semantic error while remaining indistinguishable from correct trajectories in the chosen hidden representation, no unsupervised geometry-only detector can be guaranteed to find it.

\subsection{Notation and transport identities}

For a distribution $P$ on $\R^d$ and a matrix $U\in\R^{d\times k}$, $U_\#P$ denotes the pushforward distribution of $U^\top X$ for $X\sim P$.  For a positive semidefinite matrix $A$, define the squared optimal-transport cost with ground cost $c_A(x,y)=(x-y)^\top A(x-y)$ by
\begin{equation}
  \cW_{2,A}^2(P,Q)=\inf_{\pi\in\Pi(P,Q)}\E_{(X,Y)\sim\pi}(X-Y)^\top A(X-Y).
\end{equation}
When $A=I$, we write $\cW_2^2$.

\begin{lemma}[Point-to-cloud transport]
\label{lem:point-cloud}
Let $Q$ be any distribution on $\R^m$ with finite second moment and mean $\mu_Q$, and let $A\succeq0$.  Then, for any $x\in\R^m$,
\begin{equation}
  \cW_{2,A}^2(\delta_x,Q)
  =\E_{Y\sim Q}(x-Y)^\top A(x-Y)
  =(x-\mu_Q)^\top A(x-\mu_Q)+\Tr(A C_Q),
  \label{eq:point-cloud}
\end{equation}
where $C_Q$ is the covariance of $Q$.
\end{lemma}

\begin{proof}
The only coupling between the point mass $\delta_x$ and $Q$ is the law of $(x,Y)$ with $Y\sim Q$.  This gives the first equality.  For the second equality, write $Y=\mu_Q+\xi$ with $\E\xi=0$ and $\E\xi\xi^\top=C_Q$:
\[
\E(x-Y)^\top A(x-Y)=(x-\mu_Q)^\top A(x-\mu_Q)+\E\xi^\top A\xi,
\]
where the cross term vanishes and $\E\xi^\top A\xi=\Tr(A C_Q)$.
\end{proof}

The detector score in the main text is the special case in which $x$ is the augmented transition vector
\[
\phi_t(U)=\big[z_t,\Delta z_t,\Delta^2z_t\big],\qquad
z_t=U^\top\tilde h_t,
\]
with $\Delta z_t=z_t-z_{t-1}$ and $\Delta^2z_t=z_t-2z_{t-1}+z_{t-2}$.  Lemma~\ref{lem:point-cloud} shows that this score is a quadratic deviation from the correct-transition cloud.  If $A$ is block diagonal, the score is a weighted sum of position, velocity, and acceleration deviations.  If $A$ has off-diagonal blocks, the score also includes directional-persistence terms such as $\langle \Delta z_t,\Delta z_{t-1}\rangle$ after expanding the quadratic form.  Thus the hand-designed features used by GeoReason are a low-order coordinate system for a learned transport cost.

\subsection{Proof of the contrastive transport theorem}

\begin{theorem}[Contrastive transport projection]
\label{thm:app-ctpca}
Let $X_0\sim P_0$ and $X_1\sim P_1$ be hidden vectors in $\R^d$ with means $\mu_0,\mu_1$ and covariances $C_0,C_1$.  For $U\in\R^{d\times k}$ with $U^\top U=I_k$, define
\begin{equation}
  \Gamma(U)=\E_{X_1}\cW_2^2(\delta_{U^\top X_1},U_\#P_0)
     -\E_{X_0}\cW_2^2(\delta_{U^\top X_0},U_\#P_0).
\end{equation}
Then
\begin{equation}
  \Gamma(U)=\Tr\!\big[U^\top M U\big],
  \qquad
  M=(\mu_1-\mu_0)(\mu_1-\mu_0)^\top+C_1-C_0.
\end{equation}
The maximizers of $\Gamma(U)$ are the top-$k$ eigenspaces of $M$, and the maximum value is $\sum_{i=1}^k\lambda_i(M)$, where $\lambda_1(M)\geq\cdots\geq\lambda_d(M)$.
\end{theorem}

\begin{proof}
Let $Y_0\sim P_0$ be independent of $X_0,X_1$.  By Lemma~\ref{lem:point-cloud},
\begin{align*}
\E_{X_1}\cW_2^2(\delta_{U^\top X_1},U_\#P_0)
&=\E\|U^\top(X_1-Y_0)\|_2^2 \\
&=\E\Tr\big[U^\top(X_1-Y_0)(X_1-Y_0)^\top U\big] \\
&=\Tr\big[U^\top\{C_1+C_0+(\mu_1-\mu_0)(\mu_1-\mu_0)^\top\}U\big].
\end{align*}
Similarly, for an independent copy $Y_0'$ of $X_0$,
\begin{align*}
\E_{X_0}\cW_2^2(\delta_{U^\top X_0},U_\#P_0)
&=\E\|U^\top(X_0-Y_0')\|_2^2
=2\Tr(U^\top C_0U).
\end{align*}
Subtracting gives $\Gamma(U)=\Tr(U^\top M U)$.  Maximization over orthonormal $U$ is exactly the Ky Fan variational problem, whose solutions are the top-$k$ eigenspaces of $M$.
\end{proof}

\paragraph{Connection to cPCA.}
If trace-normalization removes the first-order mean shift, then $\mu_1\approx\mu_0$ and $M\approx C_1-C_0$, the standard contrastive covariance matrix.  A background penalty $\alpha C_0$ corresponds to optimizing
\[
\Gamma_\alpha(U)=\Tr\{U^\top[(\mu_1-\mu_0)(\mu_1-\mu_0)^\top+C_1-\alpha C_0]U\},
\]
which favors directions where first-error variance or displacement is large relative to correct-step variance.  This is the theoretical objective behind the label-conditioned teacher projection.

\subsection{Finite-sample stability of the contrastive projection}

The teacher uses empirical means and covariances.  The following proposition records the stability needed for the main theorem to remain meaningful with finite traces.

\begin{proposition}[Finite-sample gap preservation]
\label{prop:finite-sample}
Let $\widehat M$ be the empirical version of $M$ and suppose $\|\widehat M-M\|_{\op}\leq\varepsilon_M$.  Let $U_\star$ be a top-$k$ eigenspace of $M$ and $\widehat U$ a top-$k$ eigenspace of $\widehat M$.  Then
\begin{equation}
  \Gamma(\widehat U)
  \geq \Gamma(U_\star)-2k\varepsilon_M.
  \label{eq:gap-preservation}
\end{equation}
If additionally $\lambda_k(M)-\lambda_{k+1}(M)=\Delta>2\varepsilon_M$, then the subspace error obeys
\begin{equation}
  \|\sin\Theta(\widehat U,U_\star)\|_{\op}
  \leq \frac{2\varepsilon_M}{\Delta}.
  \label{eq:davis-kahan}
\end{equation}
\end{proposition}

\begin{proof}
Since $\widehat U$ maximizes $\Tr(U^\top\widehat M U)$ over orthonormal $U$,
\[
\Tr(\widehat U^\top\widehat M\widehat U)\geq
\Tr(U_\star^\top\widehat M U_\star).
\]
Using $|\Tr(U^\top(\widehat M-M)U)|\leq k\varepsilon_M$ for any orthonormal $U$ gives
\[
\Tr(\widehat U^\top M\widehat U)
\geq \Tr(\widehat U^\top\widehat M\widehat U)-k\varepsilon_M
\geq \Tr(U_\star^\top\widehat M U_\star)-k\varepsilon_M
\geq \Tr(U_\star^\top M U_\star)-2k\varepsilon_M.
\]
This is Eq.~\eqref{eq:gap-preservation}.  Eq.~\eqref{eq:davis-kahan} follows from the Davis-Kahan sin-theta theorem applied to $M$ and $\widehat M$.
\end{proof}

\paragraph{A typical concentration rate.}
If the hidden vectors in both classes are sub-Gaussian with parameter $\kappa$ and sample sizes $n_0,n_1$, then standard covariance concentration yields, with probability at least $1-\delta$,
\begin{equation}
\label{eq:sample-rate}
  \varepsilon_M
  \lesssim
  \kappa^2\sqrt{\frac{d+\log(1/\delta)}{n_0\wedge n_1}}
  +\kappa\|\mu_1-\mu_0\|_2\sqrt{\frac{d+\log(1/\delta)}{n_0\wedge n_1}},
\end{equation}
up to universal constants and lower-order terms.  The first term is covariance estimation; the second arises from the rank-one mean-shift component.  In practice the effective dimension is the layer-wise intrinsic rank after trace normalization, which is often much smaller than the raw hidden width.

\subsection{Proof of first-error localization}

\begin{theorem}[Localization under a transport margin]
\label{thm:app-localization}
Let $\tau$ be the first error.  Suppose that, for $t<\tau$,
\begin{equation}
  \Prob\{S(t)-\mu_c\geq u\}
  \leq \exp\{-c\min(u^2/\nu^2,u/b)\}
  \quad\text{for all }u>0,
  \label{eq:correct-tail}
\end{equation}
and that
\begin{equation}
  \Prob\{S(\tau)\geq \mu_c+\gamma\}\geq1-\beta.
  \label{eq:margin-event}
\end{equation}
Assume the empirical score satisfies
\begin{equation}
  \Prob\left\{\max_{t\leq\tau}|\widehat S(t)-S(t)|\leq\gamma/4\right\}
  \geq1-\alpha.
  \label{eq:estimation-event}
\end{equation}
Set $\theta=\mu_c+\gamma/2$ and $\widehat\tau=\min\{t:\widehat S(t)\geq\theta\}$.  Then
\begin{equation}
  \Prob\{\widehat\tau=\tau\}
  \geq 1-\alpha-\beta-(\tau-1)
  \exp\!\left[-c\min\!\left(\frac{\gamma^2}{16\nu^2},\frac{\gamma}{4b}\right)\right].
\end{equation}
\end{theorem}

\begin{proof}
Let $E_{\mathrm{est}}$ be the event in Eq.~\eqref{eq:estimation-event}, and let $E_{\mathrm{err}}$ be the event in Eq.~\eqref{eq:margin-event}.  On $E_{\mathrm{est}}\cap E_{\mathrm{err}}$,
\[
\widehat S(\tau)\geq S(\tau)-\gamma/4\geq \mu_c+3\gamma/4>\theta,
\]
so the first error is detected.  A false alarm before $\tau$ can occur on $E_{\mathrm{est}}$ only if, for some $t<\tau$,
\[
S(t)\geq \widehat S(t)-\gamma/4\geq \theta-\gamma/4=\mu_c+\gamma/4.
\]
By the union bound and Eq.~\eqref{eq:correct-tail},
\[
\Prob\{\exists t<\tau:S(t)\geq\mu_c+\gamma/4\}
\leq (\tau-1)\exp\left[-c\min\left(\frac{\gamma^2}{16\nu^2},\frac{\gamma}{4b}\right)\right].
\]
Combining with the failure probabilities of $E_{\mathrm{est}}$ and $E_{\mathrm{err}}$ gives the result.
\end{proof}

\paragraph{Interpretation.}
The theorem does not require all post-error steps to remain anomalous.  This matters for the empirical phenomenon in which the trajectory may jump at the first error and then return toward the correct region.  The proof only needs the first error to cross the transport threshold before any earlier correct step does.

\subsection{Distillation transfer}

\begin{proposition}[Teacher-student decision preservation]
\label{prop:app-distill}
Let $S_T(t)$ and $S_S(t)$ be teacher and student scores on a trace, with a common threshold $\theta$.  Let
\[
  m_T=\min_{1\leq t\leq T}|S_T(t)-\theta|.
\]
If $\max_t|S_S(t)-S_T(t)|\leq\varepsilon<m_T$, then the teacher and student assign identical binary labels to all steps and therefore return the same first-error index.  For random traces,
\begin{equation}
  \Prob\{\widehat\tau_S\neq\widehat\tau_T\}
  \leq \Prob\{m_T\leq\varepsilon\}
   +\Prob\{\max_t|S_S(t)-S_T(t)|>\varepsilon\}.
\end{equation}
\end{proposition}

\begin{proof}
If $|S_S(t)-S_T(t)|<|S_T(t)-\theta|$ for every $t$, then $S_S(t)-\theta$ and $S_T(t)-\theta$ have the same sign for every $t$.  Hence every step label is identical, and the first threshold crossing is identical.  The probabilistic statement is the complement of this deterministic event.
\end{proof}

This proposition clarifies why a student can match the teacher in-domain yet fail under dataset or model shift.  The teacher is built from the contrastive transport matrix, so it depends on a low-dimensional instability direction.  The student observes full hidden states.  If nuisance directions vary across datasets or LLM families, the approximation error $\max_t|S_S(t)-S_T(t)|$ can increase.  Alternatively, if the teacher scores concentrate near threshold on the shifted domain, $m_T$ shrinks.  Either mechanism breaks decision preservation even when average regression loss appears acceptable.

\subsection{Relation to the seven GeoReason features}

The theoretical score in Eq.~\eqref{eq:point-cloud} is quadratic in the augmented transition variables.  Expanding with a block matrix
\[
A=\begin{bmatrix}
A_{00}&A_{01}&A_{02}\\
A_{10}&A_{11}&A_{12}\\
A_{20}&A_{21}&A_{22}
\end{bmatrix}
\]
gives terms of the form
\begin{align*}
&z_t^\top A_{00}z_t,\quad
\Delta z_t^\top A_{11}\Delta z_t,
\quad
\Delta^2z_t{}^\top A_{22}\Delta^2z_t,\\
&z_t^\top A_{01}\Delta z_t,
\quad
\Delta z_t^\top A_{12}\Delta^2z_t,
\quad
z_t^\top A_{02}\Delta^2z_t,
\end{align*}
plus linear and constant terms from the correct-transition mean.  The scalar features used by GeoReason can be interpreted as computationally cheap summaries of these quantities:
\begin{itemize}
  \item projected state $z_t$ and magnitude $\|z_t\|$ estimate position in the contrastive transport lens;
  \item normalized magnitude estimates point-to-cloud distance after trace-wise scale correction;
  \item velocity $\|\Delta z_t\|$ and acceleration $\|\Delta^2z_t\|$ estimate local transition cost;
  \item rolling energy estimates a local average of $S(t)$, reducing false positives from isolated noise;
  \item directional persistence estimates cross terms between successive increments, distinguishing coherent progress from erratic jumps.
\end{itemize}
The MLP is therefore not required to invent a new geometric statistic from scratch; it learns a nonlinear calibration of a transport-motivated sufficient feature family.

\subsection{Scope of the guarantee}

The results above rely on three substantive assumptions.  First, the hidden layer must encode reasoning correctness through a contrastive transport direction; if $P_0$ and $P_1$ are identical after projection, no geometry-based detector can separate them.  Second, the first error must have a margin $\gamma$ larger than normal correct-step fluctuations; very subtle errors may be detectable only with additional semantic supervision.  Third, deployable performance requires the student to preserve the teacher margin under distribution shift.  These assumptions match the empirical structure of GeoReason: the teacher is a diagnostic upper bound for the hidden-space signal, while the student measures how much of that signal can be recovered without inference-time labels.

\subsection{Computational cost and implementation details}
\label{subsec:complexity}

For $N=\sum_i m_i$ total steps, hidden dimension $d$, and cPCA rank $k$, naive covariance construction costs $O(Nd^2)$.  Our implementation uses a matrix-free randomized eigensolver for Eq.~\eqref{eq:contrastive-matrix-method}, requiring $O(Ndk)$ time and $O(dk)$ working memory after streaming the normalized hidden states.  Feature extraction costs $O(Nk)$ and \teach{teacher} training is negligible relative to hidden-state extraction.  \stud{Student} inference costs one LLM forward pass to obtain hidden states plus $O(mH^2)$ for a BiLSTM with hidden width $H$ on a trace of length $m$.  No step requires sampling multiple completions or querying an external verifier, which distinguishes GeoReason from self-consistency and process-supervision pipelines.


\newpage

\end{document}